%% file: MPLB.tex
\documentclass[a4paper, 10pt, conference]{ieeeconf}      

\IEEEoverridecommandlockouts    

\usepackage{xspace}
\usepackage{amsmath,amssymb,amsfonts}
\usepackage{wrapfig}
\usepackage{graphicx}
\usepackage[font=footnotesize, labelfont={sf}]{caption}
\usepackage{subfig}	
\usepackage[switch, columnwise]{lineno}
\usepackage{color}
\usepackage{algorithm}
\usepackage[noend]{algorithmic}
\usepackage{framed}
\usepackage{listings}	%used for code listings
\usepackage{caption}

\usepackage[]{epsfig,color}
\usepackage{multirow}
\usepackage{fancybox}
\usepackage{url}

\input{tex/macros}
\textVersion{\setlength\textfloatsep{0.6\baselineskip plus 3pt minus 2pt}}{}

\begin{document}
%\linenumbers
\pagestyle{plain}
\pagenumbering{arabic} 

%%%%%%%%%%%%%%%%%%%%%%%%%%%%%%%%%%%%%%%%%%%%%%%%%%%%%%%%%%%%%%%%%%%%%%%%%%%%%%
%  Title
%%%%%%%%%%%%%%%%%%%%%%%%%%%%%%%%%%%%%%%%%%%%%%%%%%%%%%%%%%%%%%%%%%%%%%%%%%%%%%
%  Title
\title
{\LARGE \bf
Asymptotically-Optimal Motion Planning using Lower Bounds on Cost}

\author{Oren Salzman and Dan Halperin$^*$% <-this % stops a space
\thanks{
$^*$
Blavatnik School of Computer Science,
Tel-Aviv University, Israel}% <-this % stops a space
\thanks{
This work has been supported in part 
by the Israel Science Foundation (grant no. 1102/11),
by the German-Israeli Foundation (grant no. 1150-82.6/2011), and
by the Hermann Minkowski--Minerva Center for Geometry at Tel Aviv 
University.}}%

\maketitle
\thispagestyle{empty}
\pagestyle{empty}

%%%%%%%%%%%%%%%%%%%%%%%%%%%%%%%%%%%%%%%%%%%%%%%%%%%%%%%%%%%%%%%%%%%%%%%%%%%%%%%
%  Sections
%%%%%%%%%%%%%%%%%%%%%%%%%%%%%%%%%%%%%%%%%%%%%%%%%%%%%%%%%%%%%%%%%%%%%%%%%%%%%%
\input {tex/abstract}
\IEEEpeerreviewmaketitle

\input {tex/intro}
\input {tex/related_work}
\input {tex/alg}

\input {tex/analysis}
\input {tex/near_optimal}
\input {tex/Evaluation}
\input {tex/Discussion}

\input {tex/thanks}

%%%%%%%%%%%%%%%%%%%%%%%%%%%%%%%%%%%%%%%%%%%%%%%%%%%%%%%%%%%%%%%%%%%%%%%%%%%%%%
%  Bibliography
%%%%%%%%%%%%%%%%%%%%%%%%%%%%%%%%%%%%%%%%%%%%%%%%%%%%%%%%%%%%%%%%%%%%%%%%%%%%%%
\bibliographystyle{IEEEtran}
\bibliography{bibliography}

%%%%%%%%%%%%%%%%%%%%%%%%%%%%%%%%%%%%%%%%%%%%%%%%%%%%%%%%%%%%%%%%%%%%%%%%%%%%%%
%  END document
%%%%%%%%%%%%%%%%%%%%%%%%%%%%%%%%%%%%%%%%%%%%%%%%%%%%%%%%%%%%%%%%%%%%%%%%%%%%%%
\end{document}

%% file: tex/macros.tex
\newcommand{\ignore}[1]{}

\newboolean{ShowTODO}
\setboolean{ShowTODO}{true}
\newcommand{\todo}[1]{%
  \ifthenelse{\boolean{ShowTODO}}%
  {{\color{cyan} #1}}%
  {}%
}

\newcommand{\Cpp}{C\raise.08ex\hbox{\tt ++}\xspace}

\newcommand{\Cfree}{\ensuremath{\calX_{\rm free}}\xspace}

\newcommand{\Cs}{C-space\xspace}
\newcommand{\Css}{C-spaces\xspace}
\newcommand{\LB}{MPLB\xspace}
\newcommand{\g}{cost-to-come}
\newcommand{\h}{cost-to-go}
\newcommand{\clb}{\ensuremath{c_{i-1}(\text{\LB})}\xspace}
\newcommand{\clbi}{\ensuremath{c_{i}(\text{\LB})}\xspace}
\newcommand{\cfmt}{\ensuremath{c_i(\text{aFMT}^*)}\xspace}

\newcommand{\set}[1]{\ensuremath{\{ #1\}}}

\newcommand{\calG}{\ensuremath{\mathcal{G}}\xspace}

\newcommand{\calX}{\ensuremath{\mathcal{X}}\xspace}

\newcommand{\calT}{\ensuremath{\mathcal{T}}\xspace}
\newcommand{\calH}{\ensuremath{\mathcal{H}}\xspace}

\newtheorem{thm}{Theorem}
\newtheorem{cor}{Corollary}
\newtheorem{lem}{Lemma}
\newtheorem{obs}{Observation}
\newtheorem{definition}[thm]{Definition}

\newboolean{ICRA}
\newboolean{ARXIV}

\setboolean{ICRA}{true}
\ifthenelse{\boolean{ICRA}}
	{\setboolean{ARXIV}{false} }
	{\setboolean{ARXIV}{true}  }
\newcommand{\textVersion}[2]
{\ifthenelse{\boolean{ICRA} }{#1}{}\ifthenelse{\boolean{ARXIV}}{#2}{}}

\ifthenelse{\boolean{ICRA} }
{

}
{
}

%% file: tex/abstract.tex
\begin{abstract}
Many path-finding algorithms on graphs such as~A* are sped up by using a heuristic function that gives lower bounds on the cost to reach the goal.
Aiming to apply similar techniques to speed up sampling-based motion-planning algorithms, 
we use effective lower bounds on the cost between configurations
to tightly estimate the cost-to-go.
We then use these estimates in an anytime asymptotically-optimal algorithm which we call Motion Planning using Lower Bounds (\LB).
\LB is based on the Fast Marching Trees (FMT*) algorithm~\cite{JP13} recently presented by Janson and Pavone.
An advantage of our approach is that in many cases 
(especially as the number of samples grows) 
the weight of collision detection in the computation is almost negligible with respect to nearest-neighbor calls.
We prove that \LB performs \textbf{no more} collision-detection calls than an anytime version of FMT*.
Additionally, we demonstrate in simulations that for certain scenarios, the algorithmic tools presented here enable efficiently  producing low-cost paths while spending only a small fraction of the running time on collision detection.
%\textcolor{blue}{We prove that \LB performs \textbf{no more} collision-detection calls than an anytime version of FMT* (called aFMT*)
%and demonstrate in simulations that for certain scenarios, 
%\LB produces lower-cost paths faster (by a factor of between two and three) han  aFMT*.
%This is done while spending a fraction of the running time on collision detection.}
\end{abstract}

%% file: tex/intro.tex
\section{Introduction}
Motion-planning algorithms aim to find a collision-free path for a robot moving amidst obstacles. 
\textVersion
{%
	The most prevalent approach in practice is to use sampling-based techniques~\cite{CBHKKLT05}.%
}
{%
	The general problem is PSPACE-hard when the number of degrees of freedom (DoF) is part of the input~\cite{R79}.
	Thus, the most prevalent approach in practice is to use sampling-based techniques and relaxing completeness to probabilistically completeness~\cite{CBHKKLT05}.%
}
These algorithms sample points in the robot's \emph{configuration-space} (\Cs) and connect close-by configurations to construct a graph called a \emph{roadmap}.
Often, a \emph{low-cost} path is desired, where cost can be measured in terms of, for example, 
\textVersion
{path length or energy consumption along the path.}
{path length, path clearance or energy consumption along the path.}

Sampling-based algorithms rely on two central primitive operations:  
\emph{Collision Detection (CD)}
and 
\emph{Nearest Neighbors (NN)} search.
CD determines whether a configuration is collision-free or not and is often used to assess if a path connecting close-by configurations is collision-free. This latter operation is referred to as \emph{Local Planning}~(LP).
An NN data structure preprocesses a set of points to efficiently answer queries such as 
\textVersion
{``which are the points within radius~$r$ of a given query point?''}
{``which is the closest point'' or ``which are the points within radius $r$'' of a query point?}
In practice, the cost of CD, primarily due to LP calls, often dominates the running time of  sampling-based algorithms, and is typically regarded as the computational bottleneck for such algorithms. 
For a summary of the computational complexity of NN, CD and LP in sampling-based motion-planning algorithms see~\cite{BKOF12}.

\textVersion
{}
{Although sampling-based algorithms have appeared in the literature since the mid 90's (see, e.g.,~\cite{HLM99, KSLO96, KL00} to mention just a few),
only recently an algorithm  that has guarantees on the cost of the produced path has been suggested.}
In their influential work, Karaman and Frazzoli~\cite{KF11} analyzed existing 
sampling-based algorithms (namely PRM~\cite{KSLO96} and RRT~\cite{KL00}) and introduced the notion of \emph{asymptotic optimality} (AO);
an algorithm is said to be AO if the cost of the solution produced by it converges to the cost of the optimal solution if the algorithm is run for sufficiently long time. 
They proposed AO variants of PRM and RRT called PRM* and RRT*, respectively.
\textVersion
{  }
{
In these latter variants, each node in the roadmap should consider connections to all nodes within a neighborhood of radius proportional to $\left( \frac{\log n}{n} \right)^{\frac{1}{d}}$, where $n$ is the number of collision-free samples used by the algorithm and $d$ is the dimension of the \Cs. 
}%
However, the AO of PRM* and RRT*  comes at the cost of increased running time and memory consumption when compared to their non-optimal counterparts.
To reduce this cost, several improvements were proposed which 
modify 
\textVersion{the sampling scheme~\cite{AS11, GSB14},}{the sampling scheme~\cite{AS11, GSB14, INMAH12},} 
the CD~\cite{BKOF12}, 
or relax the optimality to \emph{asymptotic near-optimality} (ANO)~\cite{DB14, LLB13, SH13}.
An algorithm is said to be ANO if, given an approximation factor $\varepsilon$, the cost of a solution returned by the algorithm is guaranteed to converge to within a factor of $1+\varepsilon$ of the cost of the optimal  solution.

Following the introduction of PRM* and RRT*, other AO algorithms have been suggested.%
\textVersion
{
RRT\#~\cite{AT13}, 
extends its roadmap in a similar fashion to RRT*.
However, in contrast to RRT* which only performs \emph{local} rewiring,
RRT\# efficiently propagates changes to \emph{all} the relevant parts of the roadmap.%
}
{
The first, by Arslan and Tsiotras~\cite{AT13} 
borrowed ideas used from the Lifelong Planning A* algorithm~\cite{KLF04}. 
They suggest RRT\# (RRT "sharp"), which also guarantees AO, but, in addition,  ensures that the constructed spanning tree rooted at the initial state contains lowest-cost path information for vertices
which have the potential to be part of the optimal solution.
RRT\# extends its roadmap in a similar fashion to RRT* but adds a replanning procedure.
Thus, in contrast to RRT* which only performs \emph{local} rewiring of the search tree,
RRT\# efficiently propagates changes to \emph{all} the relevant parts of the roadmap.
}
\textVersion{}
{

}
Another AO algorithm, proposed by Janson and Pavone, is the \emph{Fast Marching Trees} (FMT*)~\cite{JP13} algorithm.
\textVersion
{
FMT*, reviewed in detail in Section~\ref{sec:background}, was shown to converge to an optimal solution faster than PRM* or RRT*.
}
{
FMT* is shown to converge to an optimal solution faster than PRM* or RRT*. 
It uses a set of probabilistically-drawn configurations to construct a tree, which grows  in cost-to-come space (see Section~\ref{sec:background} for more details).
}

\vspace{2mm}
\noindent
\textbf{Contribution and paper organization}
We show how by looking at the roadmap induced by a set of samples, we can compute \emph{effective lower bounds} on the cost-to-go of nodes.
This is done without performing expensive LP calls and allows to efficiently guide the search performed by the algorithm. 
We call our scheme Motion Planning using Lower Bounds or \LB for short.
%We show how by looking at the roadmap induced by a set of samples without performing expensive calls to the LP, we can compute \emph{effective lower bounds} on the cost-to-go of nodes.
%These lower bounds allow to efficiently guide the search of FMT* and reduce the amount of LP calls. 
%We call our scheme Motion Planning using Lower Bounds or \LB for short.

%The contribution of this paper is twofold:
%First, we show how by looking at the roadmap induced by a set of samples without performing expensive calls to the LP, we can compute \emph{effective lower bounds} on the cost-to-go of nodes in the roadmap.
%These lower bounds allow to efficiently guide the search of the FMT* algorithm and reduce the amount of LP calls. 
%We call our scheme Motion Planning using Lower Bounds or \LB for short.
%Second, we show that by introducing an 
%\emph{approximation factor~$\varepsilon$}, 
%we ensure that calls to the expensive local planner will occur only between nodes that have the potential to improve the solution at hand by a large factor.
%This approach resembles the work by Alterovitch et al.~\cite{APD11} who propose a lazy variant of RRT* to decrease the number of collision checks while still retaining optimality guarantees.
%Their work, though, relies on~a user-provided parameter and 
%a \textVersion{scheme}{general scheme}  for selecting this parameter is not specified. 

An interesting and useful implication of our approach is that the weight of CD and LP becomes negligible when compared to that of the NN calls\footnote{Throughout this paper, when we say an NN call we mean a call to find all the nodes within radius~$r(n)$ of a given node. $r(n)$ is a radius depending on the number of samples used and will be formally defined in Section~\ref{sec:background}.}.
Bialkowski et al.~\cite{BKOF12} introduced a technique which replaces CD calls by NN calls. Their scheme relies on additional data produced by the CD algorithm used, 
namely, a bound on the clearance of a configuration (if it is collision free) or on its penetration depth (if it is not collision free).
Alas, such a bound is not trivial to compute for some prevalent \Css.
In contrast, our algorithmic framework can be paired with existing off-the-shelf NN, CD and LP procedures.
Thus, our results, which are more general (as they are applicable to general \Css), strengthen the conjecture of Bialkowski et al.~\cite{BKOF12} that NN computation and not CD may be the bottleneck of sampling-based motion-planning algorithms.

This work continues and expands our recent work~\cite{SH13} where we relaxed the AO of RRT* to ANO using lower bounds. 
The novel component here is that multiple nodes are processed \emph{simultaneously} while in RRT* the nodes are processed one at a time. 
This allows to efficiently compute for all nodes an estimation of the cost-to-go which in turn can be used to speed up motion-planning algorithms.

Our framework is demonstrated for the case where distance is the cost function via the FMT* algorithm which is reviewed in Section~\ref{sec:background}.
As we wish to work in an \emph{anytime} setting, we introduce in Section~\ref{sec:aFMT} a straightforward adaptation of FMT* for anytime planning which we call aFMT.
We then proceed to present \LB in Section~\ref{sec:alg}.
We analyze aFMT* and \LB with respect to the amount of calls to the NN and LP procedures in Section~\ref{sec:analysis}
%and continue to present an ANO variant of \LB in Section~\ref{sec:apx}.
%Next, we 
and report on experimental results in Section~\ref{sec:eval}.
Specifically, we demonstrate in simulations that for certain scenarios, 
\LB produces lower-cost paths faster (by a factor of between two and three) than  aFMT*.
We conclude with a discussion and suggestions for future work in Section~\ref{sec:future}.

%% file: tex/related_work.tex
\section{Terminology and algorithmic background}
\label{sec:background}
We begin by formally stating the motion-planning problem and introducing several 
\textVersion
{procedures used by the  algorithms we present.}
{standard procedures used by sampling-based algorithms.}
We continue by reviewing the FMT* algorithm.

\subsection{Problem definition and terminology}
Let $\calX$, \Cfree denote the Euclidean\footnote{
Although we describe the algorithm for Euclidean spaces, by standard techniques 
\textVersion
{(see, e.g.~\cite[Section~3.5, Section 7.1.2]{CBHKKLT05} or~\cite{K04})}
{(see, e.g.~\cite[Section~3.5, Section 7.1.2]{CBHKKLT05},~\cite[Chapters 4-5]{L06} or~\cite{K04})}
the algorithm can be applied to non-Euclidean spaces such as SE3.
However, the AO proof of FMT*, presented in~\cite{JP13}, is shown only for Euclidean spaces.} 
\Cs and free space, respectively, and $d$ the dimension of the \Cs.
Let $(\Cfree, x_{\text{init}}, \calX_{\text{goal}})$ be the motion-planning problem where:
$x_{\text{init}} \in \Cfree$ is the initial free configuration of the robot and
$\calX_{\text{goal}} \subseteq \Cfree$ is the goal region.
We will make use of the following procedures:
\texttt{sample\_free$(n)$}, a procedure returning $n$ random configurations from \Cfree;
\texttt{nearest\_neighbors}$(x,V,r)$ is a procedure that returns all neighbors of $x$ with distance smaller than $r$ within the set~$V$;
\texttt{collision\_free}$(x,y)$ tests whether the straight-line segment connecting $x$ and $y$ is contained in \Cfree;
\texttt{cost}$(x,y)$ returns the cost of the straight-line  path connecting $x$ and $y$, namely, in our case, the distance.
We consider weighted graphs $\calG = (V,E)$, 
where the weight of an edge $(x,y) \in E$ is  $\texttt{cost}(x,y)$.
Given such a graph~\calG, we  denote by $\texttt{cost}_{\calG}(x, y)$ the cost of the weighted shortest path from $x$ to~$y$.
Let $\texttt{\g}_{\calG}(x)$ be $\texttt{cost}_{\calG}(x_{\text{init}}, x)$ 
and $\texttt{\h}_{\calG}(x)$ be the minimal
$\texttt{cost}_{\calG}(x, x_{\text{goal}})$ for $x_{goal} \in \calX_{\text{goal}}$.
Namely for every node $x$, $\texttt{\g}_{\calG}(x)$ is the minimal cost to reach $x$ from  $x_\text{init}$ and $\texttt{\h}_{\calG}(x)$ is the minimal cost to reach  $\calX_{\text{goal}}$ from $x$.
%Namely for every node $x$, $\texttt{\g}_{\calG}(x)$ and $\texttt{\h}_{\calG}(x)$ is the minimal cost to reach $x$ from  $x_\text{init}$ and $\calX_{\text{goal}}$, respectively.
Additionally, let 
$B_{\calG}(x_{\text{init}},r)$, $B_{\calG}(\calX_{\text{goal}},r)$
be~the set of all nodes whose cost-to-come (respectively, cost-to-go) value on $\calG$ is smaller than $r$. 
Finally, we denote by \texttt{Dijkstra}$(G, x, c)$ an	implementation of Dijkstra's algorithm\footnote{Any other algorithm that computes the shortest path from a single source to all nodes in a graph may be used.} running on the graph $G$ from $x$ until a maximal cost of $c$ has been reached. The algorithm's implementation updates the cost to reach each node from $x$ and outputs the set of nodes traversed.

Given a set of samples $V$, and a radius $r$, we denote by $G(V,r)$ the \emph{disk graph},\footnote{The disk graph is sometimes referred to as the the \emph{neighborhood graph}.} which is the graph whose set of vertices is $V$ and two vertices $x,y \in V$ are connected by an edge if the distance between $x$ and $y$ is less than $r$.

\subsection{Fast Marching Trees (FMT*)}
FMT*, outlined in Alg.~\ref{alg:fmt},  performs a ``lazy'' dynamic programming recursion on a set of sampled configurations to grow a tree rooted at $x_{init}$~\cite{JP12}. 
The algorithm samples~$n$ collision-free nodes $V$ (line 1).
It searches for a path to~$\calX_{goal}$ by building a minimum-cost spanning tree growing in cost-to-come space (line~2 and detailed in Alg.~\ref{alg:search}).
As we explain in Section~\ref{sec:alg}, the algorithm may benefit from using 
a heuristic function estimating the cost-to-go of a node and 
a bound on the maximal length of the path that should be found.
As these are not part of the original formulation of FMT*,
we describe the search procedure of FMT* using a cost-to-go estimation of zero for each node and an unbounded  maximal path length 
(marked in red in Alg.~\ref{alg:fmt} and~\ref{alg:search}).
This will allow us to use the same pseudo-code of Alg.~\ref{alg:search} to explain the \LB algorithm in Section~\ref{sec:alg}.

The search-tree is built by maintaining two sets of nodes $H, W$ such that 
$H$ is the set of nodes added to the tree that may be expanded and $W$ is the set of nodes that have not yet been added to the tree~(Alg.~\ref{alg:search}, line~1).
It then computes for each node the set of nearest neighbors\footnote{The nearest-neighbor computation can be delayed and performed only when a node is processed but we present the batched mode of computation to simplify the exposition.} of radius~$r(n)$~(line~3).
The algorithm repeats the following process: the node $z$ with the lowest cost-to-come value is chosen from $H$ (line 4 and 16).
For each neighbor $x$ of $z$ that is not already in $H$, the algorithm finds its neighbor $y \in H$ such that the cost-to-come of $y$ added to the distance between $y$ and $x$ is minimal~(lines~7-9).
If the local path between $y$ and $x$ is free, $x$ is added to~$H$ with $y$ as its parent~(lines~10-12).
At the end of each iteration $z$ is removed from~$H$~(line~13).
The algorithm runs until a solution is found or there are no more nodes to process.

To ensure AO, the radius $r(n)$ used by the algorithm is
\begin{equation}
\label{eq:r}
r(n) = (1 + \eta) 
				\cdot 
				2 \left( \frac{1}{d}\right)^{\frac{1}{d}}
				\left( \frac{\mu (\Cfree)}{\zeta_d}\right)^{\frac{1}{d}}
				\left( \frac{\log n}{n} \right)^{\frac{1}{d}},	
\end{equation}
\noindent 
where
$\eta > 0 $ is some small constant, $\mu(\cdot)$ denotes the $d$-dimensional Lebesgue measure and $\zeta_d$ is the volume of the unit ball in the $d$-dimensional Euclidean space.
\textVersion
{}
{This value is smaller than the radius used by Karaman and Frazzoli~\cite{KF11} due to the different definition of AO used by Janson and Pavone.
Specifically, Karaman and Frazzoli~\cite{KF11} use the notion of 
\emph{convergence almost everywhere}
while Janson and Pavone~\cite{JP13} use the (weaker) notion of 
\emph{convergence in probability}.}

\begin{algorithm}[t, b]
\caption{FMT* $(x_{init}, \calX_{goal}, n)$}
\label{alg:fmt}
\begin{algorithmic}[1]
	\STATE	$V \leftarrow \set{x_{\text{init}}} \cup \texttt{sample\_free}(n)$;
%					\hspace{1mm}
					$E \leftarrow \emptyset$;
%					\hspace{1mm}
					$\calT\leftarrow (V,E)$
  \STATE	PATH $\leftarrow$ \texttt{search} 
  								$(\calT, \calX_{goal}, {\color{red} 0 , \infty})$
  				\hspace{1mm} // See Alg.~\ref{alg:search}
  \RETURN PATH \hspace{3mm}
\end{algorithmic}
\end{algorithm}

\begin{algorithm}[t, b]
\caption{search $(\calT, \calX_{goal}, {\color{red} \texttt{cost\_to\_go}, c_{max}})$}
\label{alg:search}
\begin{algorithmic}[1]
  \STATE	$W \leftarrow V \setminus \set{x_{\text{init}}}$;
  				\hspace{3mm}
					$H \leftarrow \set{x_{\text{init}}}$
  \FORALL{$v \in V$} 
  	\STATE	$N_v \leftarrow 
  					\texttt{nearest\_neighbors}(V \setminus \set{v}, v, r(n))$
  \ENDFOR
  
%  \vspace{2mm}

  \STATE	$z \leftarrow x_{\text{init}}$
	\WHILE {$z \notin \calX_{\text{Goal}}	$}
		\STATE $H_{\text{new}} \leftarrow \emptyset$;
					 \hspace{3mm}
					 $X_{\text{near}} \leftarrow W \cap N_z$
%						\hspace{3mm}
%						// \textsf{	nearest vertices of $z$ not yet added to the roadmap 
%												tree}
		
%	  \vspace{2mm}

		\FOR  {$x \in X_{\text{near}}$}
			\STATE $Y_{\text{near}} \leftarrow H \cap N_x$
							\hspace{3mm}
%							// \textsf{	nearest vertices of $x$ already added to the roadmap 
%													tree}	
			\STATE $y_{\text{min}} \leftarrow \arg \min_{y \in Y_{\text{near}}} 
								\set{\texttt{cost}_{\calT}(y) + \texttt{dist}(y,x)} $
								
%			\vspace{2mm}
			
			\IF {\texttt{collision\_free}$(y_{\text{min}}, x)$}
				\STATE $\calT.\texttt{parent}(x) \leftarrow y_{\text{min}}$
				\STATE $H_{\text{new}} \leftarrow H_{\text{new}} \cup \set{x}$;
							 \hspace{3mm}
					 		 $W \leftarrow W \setminus \set{x}$
			\ENDIF
		\ENDFOR
								
%		\vspace{2mm}
					
		\STATE $H \leftarrow (H \cup H_{\text{new}}) \setminus \set{z}$
		\IF {$H = \emptyset$}
			\RETURN FAILURE
		\ENDIF

		\STATE $z \leftarrow \arg \min_{y \in H}
								\set{\texttt{cost$_{\calT}$}(y) + 
											{\color{red}		\texttt{cost\_to\_go}(y)} }$
		{\color{red}														
		\IF {$\texttt{cost$_{\calT}$}(z) + \texttt{cost\_to\_go}(z) \geq c_{max}$ }
			\RETURN FAILURE
		\ENDIF
	}
	\ENDWHILE

	\RETURN PATH \hspace{3mm}
	 
\end{algorithmic}
\end{algorithm}						

\section{Anytime FMT* (aFMT*)}
\label{sec:aFMT}
An algorithm is said to be \emph{anytime} if it yields meaningful results even after a short time and it improves the quality of the solution as more computation time is available.
We outline a straightforward enhancement to FMT* to make it anytime.
As noted in previous work (see, e.g.,~\cite{WBC13}) one can turn a batch algorithm into an anytime one by the following general approach:
choose an initial small number of samples $n=n_0$ and apply the algorithm. 
As long as time permits, double $n$ and repeat the process. 
The total running time is less than twice that of the running time for the largest $n$.
Note that as FMT* is AO, aFMT* is also~AO.

We can further speed up this method by reusing both existing samples and connections from previous iterations.
\textVersion
{Due to lack of space, we omit these details and refer the interested reader to the extended version of our paper~\cite{SH14}.}
{
Notice that this improvement does not change the asymptotic running time.
It affects only the constants of the running time.
Assume the algorithm was run with $n$ samples and now we wish to re-run it with $2n$ samples.
In order to obtain the $2n$ random samples, we take the $n$ random samples from the previous iteration together with $n$ new additional random samples.
For each node that was used in iteration $i-1$, around half of its neighbors in iteration $i$ are nodes from iteration $i-1$ and half of its neighbors are newly-sampled nodes.
Thus, if we save the results of  calls to the local planner, we can cache them to be reused in future iterations.
}

%% file: tex/alg.tex
%\pagebreak
\section{Algorithmic framework}
\label{sec:alg}
We are now ready to present our approach to exploiting lower bounds on cost in order to speed up sampling-based motion-planning algorithms.

Given a random infinite sequence of collision-free samples 
$S = s_1, s_2 \ldots$
denote by 
$V_i(S)$ the set of the first $2^i$ elements of $S$.
Let $\calG_i(S) = G(V_i(S), r(|V_i(S)|))$ and 
let $\calH_i(S) \subseteq \calG_i(S)$ be the subgraph containing collision-free edges only 
(here $r(n)$ is the radius defined in Eq.~\ref{eq:r}).
For brevity, we omit $S$ when referring to $V_i(S), \calG_i(S)$ and  $\calH_i(S)$.
Moreover, when we compare our algorithm to the aFMT* algorithm, we do so for runs on the same random infinite sequence $S$.
Clearly, for any two nodes $x, y \in V_i$, 
$
\texttt{cost}_{\calG_i}(x, y) 
	\leq  
\texttt{cost}_{\calH_i}(x, y)
$.
Thus for any node $x \in V_i$,
$
\texttt{\h}_{\calG_i}(x) \leq \texttt{\h}_{\calH_i}(x).
$
Namely, the cost-to-go computed using the disk graph~$\calG_i$ is a lower bound on the cost-to-go that may be obtained using $\calH_i$.
We call this the \emph{lower bound property}.
For an illustration, see Fig.~\ref{fig:cspace}.

\begin{figure}[t,b]
  \centering
  	\vspace{-5mm}
	 \subfloat
   [\sf ]
   {
   	\includegraphics[height =3.0 cm]{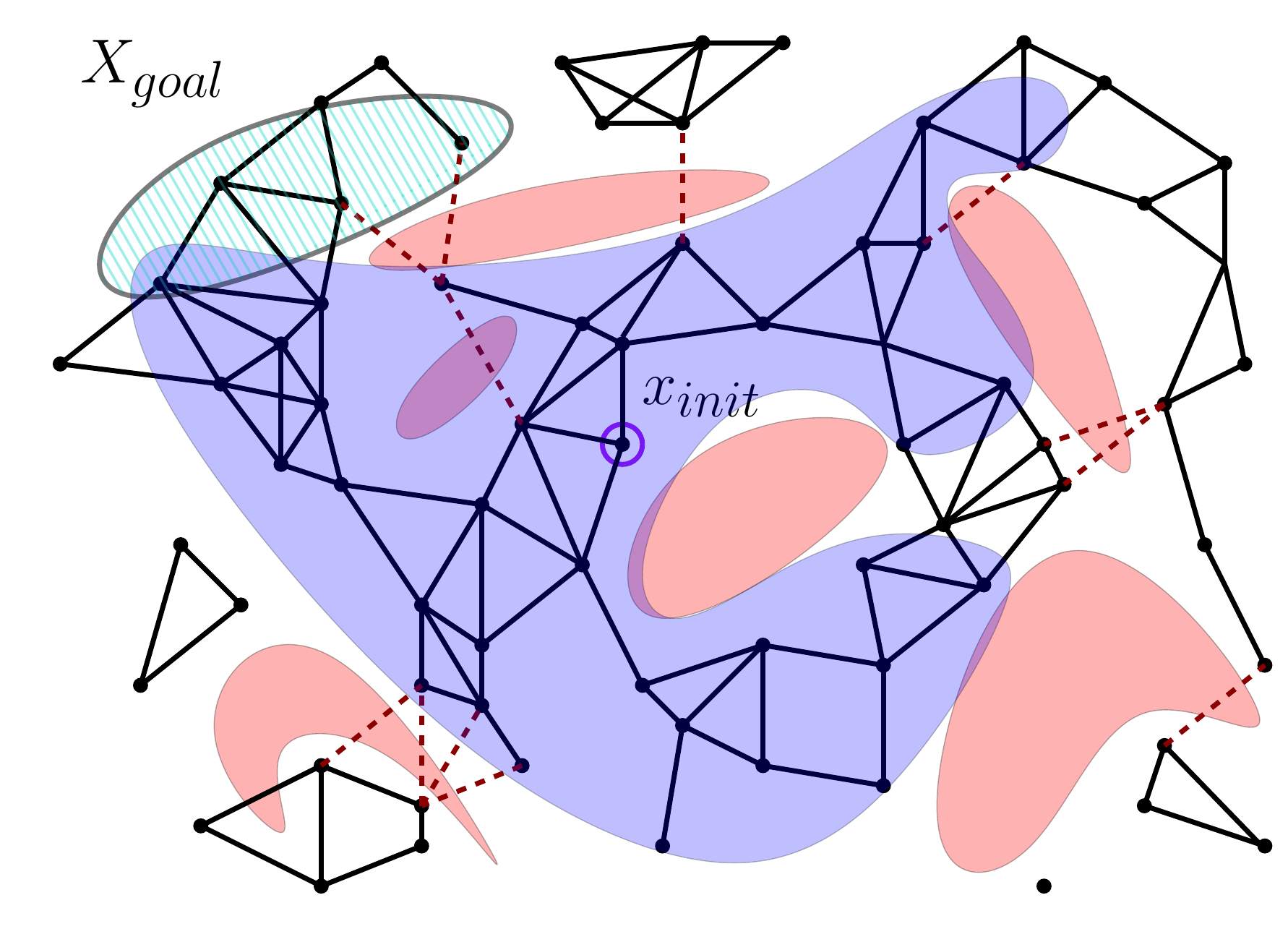}
   	\label{fig:corr_time}
   }
   \subfloat
   [\sf ]
   {
   	\includegraphics[height =3.0 cm]{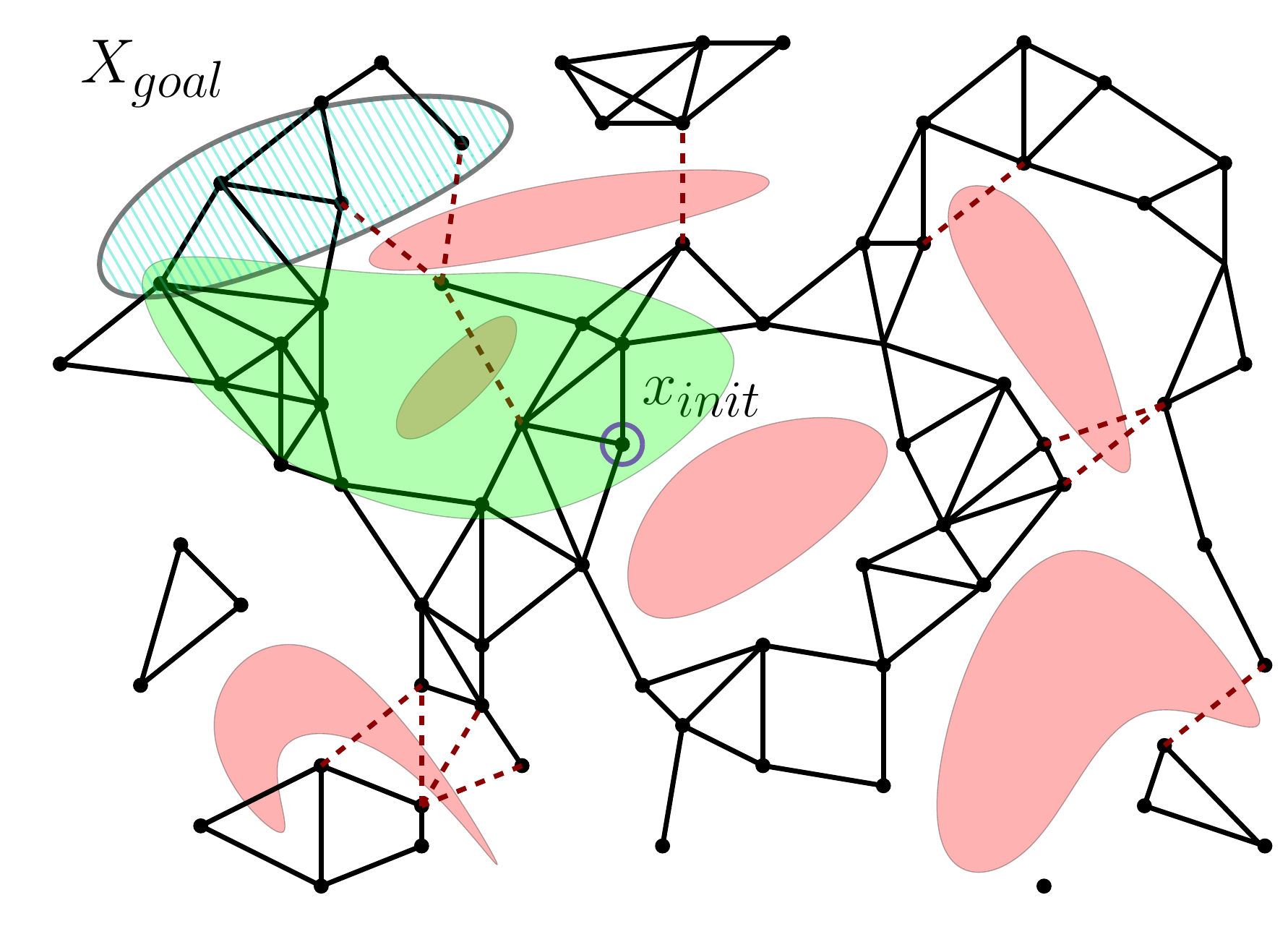}
   	\label{fig:corr_success}
   }
  \caption{\sf 	\footnotesize
  							This figure demonstrates that the part of the tree expanded 
  							when 
  							searching in	cost-to-come space 
  							(shaded blue region, Fig.~(a)) is larger than 
  							the one expanded when searching in 
  							cost-to-come+cost-to-go space 
  							(shaded green region, Fig.~(b)).
  							Obstacles in the \Cs are depicted in red, start location and 
  							goal region are depicted by a purple circle and a turquoise 
  							region, respectively.
  							Edges of the disk graph~$\calG_i$ that are 
  							contained and not contained
  							in $\calH_i$ are depicted in black and dashed red, 
  							respectively. 
  							The figure is best viewed in color.}
  \label{fig:cspace}
	\vspace{-1mm}
\end{figure}

\begin{algorithm}[b]
\caption{MPLB $(x_{init}, \calX_{goal}, n_0)$}
\label{alg:MPLB}
\begin{algorithmic}[1]
	\STATE 	$V \leftarrow \set{x_{\text{init}}}$;
					\hspace{1mm}
					$n \leftarrow n_0$;
					\hspace{1mm}
					$c_{prev} \leftarrow \infty$
	\WHILE {\texttt{time\_permits()}}
		\STATE $V \leftarrow V \cup \texttt{sample\_free}(n)$;
					\hspace{1mm}
					$E \leftarrow \emptyset$;
					\hspace{1mm}
					$\calT\leftarrow (V,E)$
		\STATE \texttt{estimate\_cost\_to\_go}$(V, x_{init}, \calX_{goal}, c_{prev})$
		\STATE PATH $\leftarrow$ \texttt{search} 
  					$(\calT, \calX_{goal}, {\color{red} \texttt{cost\_to\_go} , c_{prev}})$
		\STATE $n \leftarrow 2n$;
					 \hspace{1mm}
					 $c_{prev} = \texttt{cost}$(PATH)
	\ENDWHILE
  \RETURN PATH 
\end{algorithmic}
\end{algorithm}						

\begin{algorithm}[t, b]
\caption{estimate\_cost\_to\_go $(V, x_{init}, \calX_{goal}, c)$}
\label{alg:preproc}
\begin{algorithmic}[1]
	\STATE $V_{\text{preproc}} \leftarrow \texttt{Dijkstra} 
			(G(V, r(|V|)), x_{init}, \frac{c}{2})$
	\STATE $V_{\text{preproc}} \leftarrow V_{\text{preproc}} \cup \texttt{Dijkstra} 
			(G(V, r(|V|)), \calX_{goal}, \frac{c}{2})$
	\FOR {$x \in V \setminus V_{\text{preproc}}$}
		\STATE cost\_to\_go$(x) \leftarrow \infty $
	\ENDFOR
	\STATE $\texttt{Dijkstra} 
	(G(V_{\text{preproc}}, r(|V_{\text{preproc}}|)), \calX_{goal}, c)$
\end{algorithmic}
\end{algorithm}						

We present Motion Planning using Lower Bounds, or \LB (outlined in Alg.~\ref{alg:MPLB}). 
Similar to aFMT*, the algorithm runs in iterations and at the $i$'th iteration, uses $V_i$ as its set of samples.
Unlike aFMT*, each iteration consists of 
a \emph{preprocessing phase} (line~4) of computing a lower bound on the cost-to-go values
and a \emph{searching phase} (line~5) where a modified version of FMT* is used.

Let $c_{i}(\text{ALG})$ denote the cost of the solution obtained by an algorithm ALG using $V_i$ as the set of samples 
(set $c_{0}(\text{ALG}) \leftarrow \infty$).
We now show that only a subset of the nodes sampled in each iteration need to be considered. 
We then proceed to describe the two phases of \LB.
%\textVersion
%{}
%{In Appendix~\ref{sec:visualization} we provide a graphic demonstration
%comparing the behavior of \LB and aFMT* in different iterations. This can serve as an illustrated accompaniment to the verbal description here.}

\subsection{Promising nodes}

We use the lower bound property to consider only a \emph{subset} of $V_i$ that will be used in the $i$'th iteration.
Intuitively, we only wish to consider nodes that may produce a solution that is better than the solution obtained in previous iterations.
This leads us to the definition of promising~nodes:
\begin{definition}%{Promising node}
A node $x\!\in\!V_i$ is 
\emph{promising} 
(at iteration $i$) if 
\vspace{-3mm}
\[
	\texttt{\g}_{\calH_i}(x) + \texttt{\h}_{\calH_i}(x)\!<\!\clb.
\]
\end{definition}
\vspace{1mm}

In the preprocessing phase, \LB will traverse $\calG_i$ (and not~$\calH_i$) to collect a set of nodes that contains all promising nodes (and possibly other nodes), compute a lower bound on their cost-to-go and use this set in the searching phase.

\subsection{Preprocessing phase: Estimating the cost-to-go}
Recall that in the preprocessing phase, outlined in Alg.~\ref{alg:preproc}, we wish to compute a lower bound on the cost-to-go for (a subset of) nodes $x \in V_i$.
Specifically, the only nodes we wish to consider are \emph{promising nodes}.
This is done by collecting the set of nodes
$V_{\text{preproc}} =  
B_{\calG_i} \left( x_{\text{init}}, \frac{\clb}{2} \right)
\cup
B_{\calG_i} \left( \calX_{\text{goal}}, \frac{\clb}{2} \right)
$.
Namely, by performing 
one traversal from $x_{\text{init}}$ (line~1)
and
one traversal from $\calX_{\text{goal}}$ (line~2), 
all nodes such that 
$\texttt{\g}_{\calG_i} \leq \frac{\clb}{2}$ or
$\texttt{\h}_{\calG_i} \leq \frac{\clb}{2}$ are found.
Clearly, any node \emph{not} in either set is not promising (lines 3-4).

After collecting all nodes in $V_{\text{preproc}}$, \LB computes 
the distance of every such node from $\calX_{goal}$ (line 5).
This is done by running a shortest paths algorithm on the graph $\calG_i$ restricted to the nodes in $V_{\text{preproc}}$.
This distance is stored for each node and will be used as a lower bound on the cost-to-go.
We note that this preprocessing phase only uses NN calls and does not use any CD calls (as there are no LP calls).

\subsection{Searching phase: Using cost-to-go estimations}
The lower bounds computed in the preprocessing phase allow for two algorithmic enhancements to the searching phase when compared to aFMT*:
(i)~incorporating the cost-to-go estimation in the ordering scheme of the nodes
and
(ii)~discarding nodes that are found to be not promising.

\vspace{2mm}
\noindent
\textbf{Node ordering:}
Recall that in aFMT*, $H$ is the set of nodes added to the tree that may be expanded and that these nodes are ordered according to their cost-to-come value (Alg.~\ref{alg:search}, line~16).
Instead, we suggest using the cost-to-come added to the cost-to-go estimation to order the nodes in~$H$.
This follows exactly the formulation of A*~\cite{P84} which performs a Dijkstra-like search on a set of nodes. 
The nodes that were encountered but not processed yet ($H$ in our setting) are ordered according to a cost function $f() = g() + h()$.
Here, 
$g(x)$ is the (computed) cost-to-come value of $x$ 
($\texttt{\g}_{\calH_i}(x)$ in our case) and 
$h$ is a lower bound on the cost-to-go of $x$ to the goal
($\texttt{\h}_{\calG_i}(x)$ in our case).
aFMT* essentially uses the trivial heuristic $h = 0$.
Instead, we suggest to use a much sharper bound to speed up the search towards the goal.

\vspace{2mm}
\noindent
\textbf{Discarding nodes:}
In the preprocessing stage \LB computes a set of nodes that \emph{may} be promising, though for each such node, the cost-to-come value was not computed.
In the searching phase, once a node is added to the tree, its cost-to-come value will not change in the current iteration.
Thus, every node $x$ added to the tree with 
$
	\texttt{\g}_{\calH_i}(x) 
	 	+ 
	\texttt{\h}_{\calG_i}(x) 	
		\geq
	\clb
$
is discarded as it cannot be  promising.
This implies that  \LB will terminate an iteration when it is evident that the previous iteration's solution cannot be improved (see 
Alg.~\ref{alg:search}, lines~17-18).
%Hence, \LB will terminate an iteration when it is evident that the solution obtained in the previous iteration cannot be improved (see Alg.~\ref{alg:fmt}, lines~17-18).

In Section~\ref{sec:eval} we demonstrate through various simulations that  using lower bounds has a significant effect on the running time of the algorithm in practice.
Ordering the nodes using a heuristic that tightly estimates the cost-to-go allows \LB to expand a smaller portion of the nodes~$V_i$ while discarding nodes allows to focus the search only on nodes that may potentially improve the existing solution.

%% file: tex/analysis.tex
\section{Comparative analysis and Discussion}
\label{sec:analysis}
We compare aFMT* and \LB with respect to the size of the tree constructed in the searching phase and with respect to the primitive procedures, namely NN and LP.
\textVersion
{}
{We assume that the computational cost of the graph traversal algorithms is negligible.}
This is done by quantifying the number of NN and LP calls  performed by both algorithms and allows us to discuss the fundamental differences between the two algorithms.

%and show that \LB will perform \emph{no more} calls to the local planner than aFMT*. 
%It \emph{may} perform \emph{more} nearest-neighbors calls than aFMT*.
%\textVersion{}{The analysis of the number of LP calls is purely combinatorial. Far any fixed sequence $S$ of collision-free samples we show that any LP call that is made by \LB is also made by aFMT*.}

Let 
$\#_{\texttt{NN}, i}(\text{ALG})$,
$\#_{\texttt{LP}, i}(\text{ALG})$
denote the number of NN and LP calls  performed by an algorithm ALG in iteration~$i$, respectively for a fixed sequence of samples~$S$. 
Recall that when comparing the two algorithms, it is done for the same sequence~$S$.

\subsection{Search-tree size}
Let $V_i(\text{ALG}) \subseteq V_i$ denote the set of nodes in the tree in the $i$'th iteration of an algorithm $\text{ALG}$.
\begin{lem}
	At every iteration, the set of nodes traversed in \LB's searching phase 
	is not larger than that of aFMT*.
\end{lem}

\begin{proof}
Every node $x$ in the tree of aFMT* has cost-to-come not larger than \cfmt.
Thus, the size of the search-tree of aFMT* is:
$
	|V_{i}(\text{aFMT*})| = \\
%							B_{\calH_i} (x_{\text{init}}, \cfmt)
%						=
							|\set{	x \in V_i \ | \ 
										\texttt{\g}_{\calH_i}(x) \leq \cfmt
									}|.
$

Similar  
to aFMT*,
each node $x$ traversed by \LB in the searching phase has 
$\texttt{\g}_{\calH_i}(x) +  \texttt{\h}_{\calG_i}(x) \leq \cfmt$.
Additionally, due to node discarding (see Section~\ref{sec:alg}), 
$\texttt{\g}_{\calH_i}(x) +  \texttt{\h}_{\calG_i}(x) \leq \clb$.
Thus, 
 the size of the search-tree of \LB is:
$
	|V_{i}(\text{\LB})| = \\ 
	| \{	x \in V_i \ | \ 
									\texttt{\g}_{\calH_i}(x) + \texttt{\h}_{\calG_i}(x)
										$ \\ $\leq  
							\min \{ \clb	, \cfmt \}
								\}|.
$\\
Namely, $|V_{i}(\text{\LB})| \leq |V_{i}(\text{aFMT*})|$.

%\qed
\end{proof}
\vspace{-3mm}
\subsection{Nearest neighbor calls (NN)}
\textVersion{To quantify the number of NN queries performed by each algorithm we note the following observations (explained in detail in the extended version of this paper~\cite{SH14} due to lack of space):
\begin{obs}
\label{obs:nn_fmt}
The number of NN calls performed by aFMT* can be bounded from below as follows:  
$\#_{\texttt{NN},i}(\text{aFMT}^*) \geq |V_i(\text{aFMT*})|$.
\end{obs}
\begin{obs}
\label{obs:nn_mplb}
The number of NN calls performed by \LB is: 
$
	\#_{\texttt{NN},i}(\text{\LB}) = $\\$
	\left| \left\{ x \in V_i| x
	\in\!B_{\calG_i}\!\left( x_{\text{init}}, \frac{\clb}{2} \right)
	\!\cup\!
	B_{\calG_i}\!\left( \calX_{\text{goal}}, \frac{\clb}{2} \right) \right\}\right|.
$ 
\end{obs}
\vspace{2mm}
Thus, \LB \emph{may} perform more NN queries than FMT*.
}
{
For every node in the search tree (of either aFMT* or \LB), there is an NN call (see line~6 in Algorithm~\ref{alg:search}). There may be additional NN calls for nodes that are neighbors of nodes in the search tree (see line~8 in Algorithm~\ref{alg:search}).
Thus,
\begin{obs}
\label{obs:nn_fmt}
The number of NN calls performed by aFMT* can be bounded from below as follows:  
$\#_{\texttt{NN},i}(\text{aFMT}^*) \geq |V_i(\text{aFMT*})|$.
\end{obs}
The \LB algorithm has additional NN calls due to the preprocessing stage.
More specifically, for each node traversed in the preprocessing phase, there is one NN call. Note that in the searching phase, \LB uses only nodes traversed in the preprocessing phase, hence, there will be no additional NN calls in the searching phase.
Thus,
\begin{obs}
\label{obs:nn_mplb}
The number of NN calls performed by \LB is:  \\
$
	\#_{\texttt{NN},i}(\text{\LB}) = $\\$
	\left| \left\{ x \in V_i| x
	\in\!B_{\calG_i}\!\left( x_{\text{init}}, \frac{\clb}{2} \right)
	\!\cup\!
	B_{\calG_i}\!\left( \calX_{\text{goal}}, \frac{\clb}{2} \right) \right\}\right|.
$ 
\end{obs}
\vspace{2mm}
Thus, \LB \emph{may} perform more NN queries than FMT*.
}

%For every node in the search tree (of either aFMT* or \LB), there is an NN call (see line~8 in Alg.~\ref{alg:fmt}). There may be additional NN calls for nodes that are neighbors of nodes in the search tree (see line~10 in Alg.~\ref{alg:fmt}).
%Thus,
%\begin{obs}
%\label{obs:nn_fmt}
%The number of NN calls performed by aFMT* can be bounded from below as follows:  
%$\#_{\texttt{NN},i}(\text{aFMT}^*) \geq |V_i(\text{aFMT*})|$.
%\end{obs}
%\LB has additional NN calls due to the preprocessing stage.
%Specifically, for each node traversed in the preprocessing phase, there is one NN call. 
%In the searching phase, \LB uses only nodes traversed in the preprocessing phase, hence, this second phase incurs no additional NN calls.
%Thus,
%\begin{obs}
%\label{obs:nn_mplb}
%The number of NN calls performed by \LB is: 
%$
%	\#_{\texttt{NN},i}(\text{\LB}) = $\\$
%	\left| \left\{ x \in V_i| x
%	\in\!B_{\calG_i}\!\left( x_{\text{init}}, \frac{\clb}{2} \right)
%	\!\cup\!
%	B_{\calG_i}\!\left( \calX_{\text{goal}}, \frac{\clb}{2} \right) \right\}\right|.
%$ 
%\end{obs}

\subsection{Local planning calls (LP)}
\textVersion
{The LP}
{Recall that the local planner is a procedure that tests if the straight-line segment connecting two configurations is collision-free.
The number of LP calls depends on the \Cs and is somewhat harder to quantify. Yet, the LP procedure}
will be called whenever either algorithm (aFMT* or \LB) attempts to insert a node to the search-tree (line~10 in Alg.~\ref{alg:search}).
Thus we can state the following lemma:
\begin{lem}
If \LB performs an LP call for the edge $(x, y)$ in the $i$'th iteration
then 
aFMT* will perform an LP call for the edge $(x, y)$ as well.
\end{lem}

\begin{proof}
The LP procedure will be called for every pair of nodes $x,y$ in the search tree such that:
(i)~$x,y$ are neighbors in $\calG_i$ 
(namely their distance is less than $r(|V_i|)$),
(ii)~$\texttt{\g}_{\calH_i}(x) < \texttt{\g}_{\calH_i}(y)$ 
(namely $x$ is inserted to the tree before $y$),
and
\textVersion{(iii)~the edge $(z,y)$ is not collision-free for all other neighbors $z$ of $y$ in the tree that could potentially lead to smaller cost-to-come values of $y$.}
{(iii)~the edge $(z,y)$ is not collision-free for every $z$ that is 
a neighbor of $y$ in $\calG_i$ and
$\texttt{\g}_{\calH_i}(z) + \texttt{cost}(z,y) \leq 
 \texttt{\g}_{\calH_i}(x) + \texttt{cost}(z,x)$
(namely all other neighbors $z$ of $y$ in the tree that could potentially lead to smaller cost-to-come values of $y$).
}

If \LB performs an LP call for the edge $(x, y)$% 
\textVersion{}{ in the $i$'th iteration}
then conditions (i),(ii) and~(iii) hold for the samples $x, y$ in \LB.
	To prove the lemma we show that they hold for the samples $x, y$ in aFMT*.
Condition (i) holds trivially as it is a property of the samples. 
Note that the cost-to-come of any node $z$ computed by both algorithms equals to $\texttt{\g}_{\calH_i}(z)$. 
Using this observation and  that $V_{\text{preproc}} \subseteq V_i$ (namely the nodes used by \LB is a subset those used by aFMT*), 
conditions (ii) and (iii) hold as well.
%\qed
\end{proof}
\vspace{-3mm}
\subsection{Discussion}
From the above analysis we conclude that 
\LB will perform \emph{no more} LP calls than aFMT*.
It \emph{may} perform \emph{more} NN calls than aFMT*.
As we demonstrate empirically in the Evaluation section, 
the number of NN calls that \LB performs may actually be smaller than that of aFMT*.
%\textcolor{red}{As we demonstrate empirically in the Evaluation section, 
%the number of NN calls that \LB performs is only slightly larger than that of aFMT* (and sometimes is even lower).}
Moreover, as the number of iterations increases, 
\LB performs only a tiny fraction of the number of LP calls performed by aFMT*.

%% file: tex/near_optimal.tex
\textVersion
{
}
{\section{Relaxing optimality of MPLB}
\label{sec:apx}
Asymptotically-optimal motion-planning algorithms such as MPLB often, from a certain stage of their execution,
invest huge computational resources at only slightly improving the cost of the current best existing solution. 
We aim to overcome this problem by relaxing AO to ANO.
To this end, given an \emph{approximation factor} $\varepsilon$, 
the ANO version of MPLB (termed ANO-MPLB)  maintains the following invariant: 
\[
	\clbi \leq c_{i}(\text{ANO-\LB}) \leq (1 + \varepsilon) \cdot \clbi.
\]
Namely, at each iteration, the cost of the solution obtained by ANO-MPLB is within a factor of 
$1 + \varepsilon$ 
from the solution that MPLB would obtain for the same set of samples. 
We call this the \textbf{bounded approximation invariant}.
As MPLB is AO, and if the bounded approximation invariant is indeed maintained then the following holds.

\begin{cor}
\label{cor:ano}
ANO-MPLB is asymptotically near-optimal.
\end{cor}

ANO-MPLB is implemented by simply replacing the cost of the solution produced in the previous iteration $c_{prev}$ by the value $\frac{c_{prev}}{1+\varepsilon}$ 
(see lines 4 and 5 of Alg.~\ref{alg:MPLB}).

We note that for all definitions used in Sections~\ref{sec:alg} and~\ref{sec:analysis} and the analysis presented in Section~\ref{sec:analysis} one needs to replace $c_{i-1}(MPLB)$ by $\frac{c_{i-1}(ANO-MPLB)}{1+\varepsilon}$ for ANO-MPLB.

In order to show that \LB is ANO we first show that,
\begin{lem}
\label{lem:invariant}
The bounded approximation invariant is maintained by \LB.
\end{lem}

\begin{proof}
We prove the lemma by induction over the number of iterations.

\noindent
\textbf{Induction base:} 
At the first iteration, $c_{0}(\texttt{\LB}) \leftarrow \infty$ and all nodes are promising. Thus, both the aFMT* algorithm and the \LB algorithm use the same set of nodes $V_1$.
The difference between the two algorithm are 
(i)~the order by which nodes are processed in the searching phase and 
(ii)~the possible discarding of nodes.

Note the following observation, which follows directly from the optimality proof of the A* algorithm using any admissible heuristic.
\begin{obs}
\label{obs:ordering}
Given a fixed set of nodes, namely if no nodes are discarded,
then both the old ordering scheme (using only \g) and the new ordering scheme (using \g+\h) will return the same path.
\end{obs}

We show that if $x_{\text{init}}$ is in the same connected component than any node in $\calX_{\text{goal}}$ then no nodes are discarded 
(if not, then neither aFMT* nor \LB  can return a path to $\calX_{\text{goal}}$):
Assume falsely that a node $x$ is discarded, thus 
$
	\texttt{\g}_{\calH_i}(x) 
	 + 
	\texttt{\h}_{\calG_i}(x) 	
	\geq
\frac{\clb}{1 + \varepsilon}
  = 
  \infty
$.
As 
$\texttt{\g}_{\calH_i}(x)$ 
is bounded (and thus $x$ is in the same connected component as $x_{\text{init}}$)
we have that 
$\texttt{\h}_{\calG_i}(x) = \infty$.
This implies  that 
$x$ is in a different connected component than any node in $\calX_{\text{goal}}$ in contradiction to our assumption.

\noindent
\textbf{Induction step:} 
Assume that the optimal path produced by aFMT* in the $i$'th iteration contains only promising nodes (if not, the bounded approximation invariant is maintained trivially).
Clearly, all the promising nodes are members of the set 
$V_{\text{preproc}}$ computed by \LB in the preprocessing stage  and none of them will be discarded.
Thus, by Observation~\ref{obs:ordering}, $\clbi = \cfmt$ and the approximation invariant is maintained.
\end{proof}

Using the AO of aFMT* and Lemma~\ref{lem:invariant}, Corollary~\ref{cor:ano} immediately follows.

}

%% file: tex/Evaluation.tex
\section{Evaluation}
\label{sec:eval}
We present simulations evaluating the performance of \LB as an anytime algorithm on 2, 3 and 6 dimensional \Css.
All experiments were run on a 2.8GHz Intel Core i7 processor with 8GB of memory.
The \LB and aFMT* implementations are based on the FMT* implementation provided by Pavone's research group using the Open Motion Planning Library (OMPL~0.10.2)~\cite{SMK12}.
Each result is averaged over one hundred different runs.
Scenarios and additional material are available at
%All scenarios used that are not part of the OMPL distribution are available at 
\url{http://acg.cs.tau.ac.il/projects/MPLB}.

The AO proof of FMT* (and thus of aFMT* and MPLB) relies on the fact that the \Cs is Euclidean.
Thus, we start by studying the motion of robots translating in the plane and in space 
(Fig.~\ref{fig:corr} and~\ref{fig:grids}).
Next, we continue to examine the behavior of the algorithms in SE(3) (Fig.~\ref{fig:cubicles}). 
Here the radius provided for FMT* (Eq.~\ref{eq:r}) is irrelevant due to the differences in the rotational and translational components of the \Cs.
Hence, for both aFMT* and \LB, we chose to connect each node to its $k$ NN, where $k(n)= 9 \log n$:
Karaman and Frazzoli~\cite{KF11} proposed a variant of RRG where each node is connected to its $k_{RRG}$ NN for $k_{RRG}(n) \geq 2e \log n$. 
Although this variant was analyzed for Euclidean spaces only, applying it to non-Euclidean spaces works well in practice (see, e.g.~\cite{SH13}).

\begin{figure*}[t,b]
  \centering
  \subfloat
   [\sf Corridors]
   { 
   	\includegraphics[height = 3.1cm ]{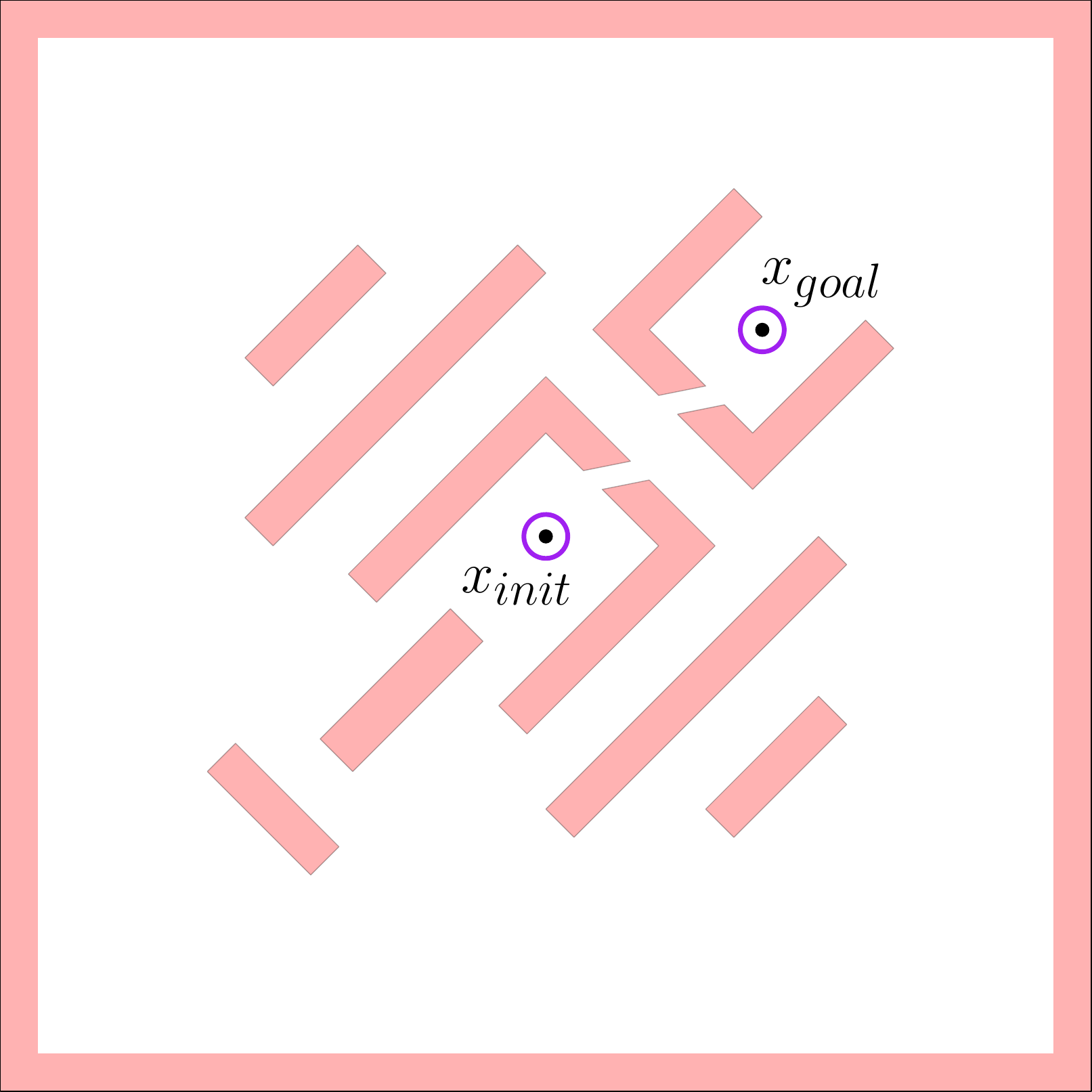}
   	\label{fig:corr}
   }
  \subfloat
   [\sf Grids]
   { 
   	\includegraphics[height = 3.1cm]{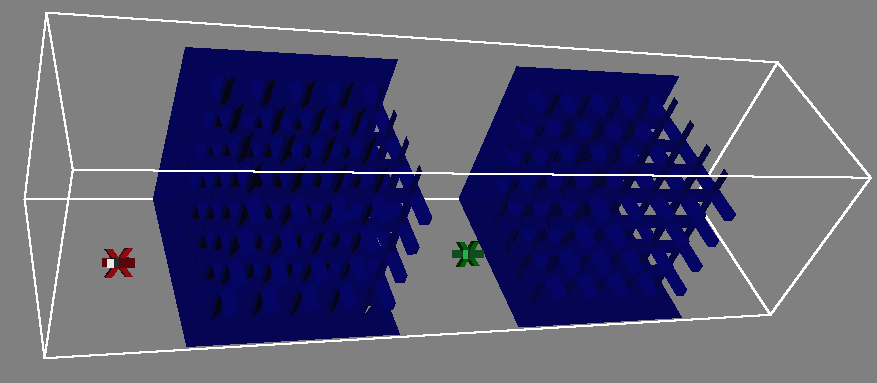}
   	\label{fig:grids}
   }
   \subfloat
   [\sf Home]
   { 
   	\includegraphics[height = 3.1cm]{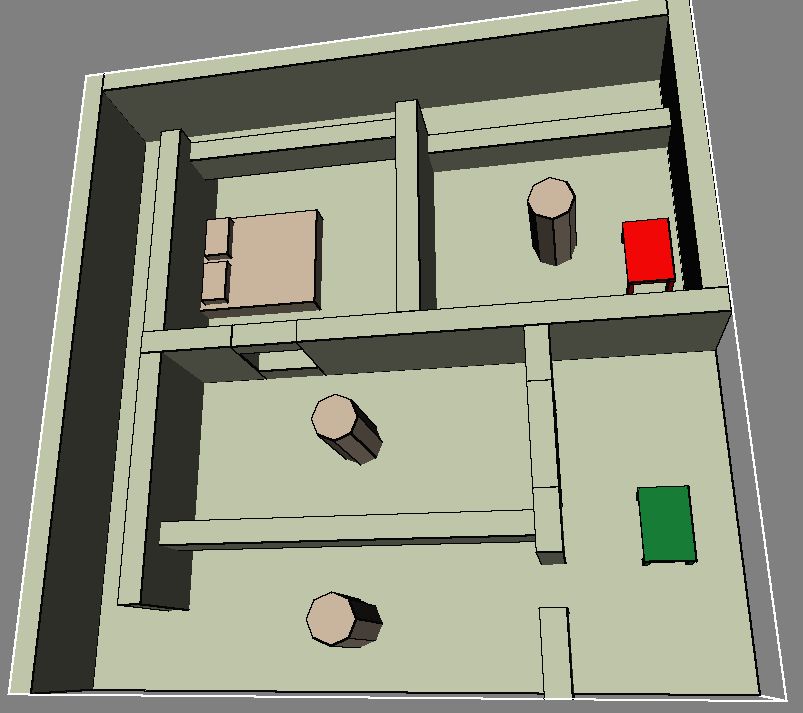}
   	\label{fig:cubicles}
   }
  \caption{\sf 	\footnotesize
  							Scenarios used for the evaluation.
  							(a) Two dimensional setting for a point robot.
  							A low-cost path is easy to find yet in order to find a 
  							high-quality path, the robot needs to pass through two narrow 
  							passages.
  							(b) Three-dimensional \Cs for a translating robot in 
  							space. To find the shortest path the robot needs to pass 
  							through a three-dimensional grid.
  							(c) Six-dimensional \Cs for an L-shaped robot 
  							translating and rotating in space. 
  							Finding a path is relatively easy yet much time is needed to 
  							converge to the optimal path.
  							Start and target configurations for (b) and (c) 
  							are depicted by green and  red robots, respectively,
  							The Home scenario is provided by the OMPL~\cite{SMK12} 
  							distribution.
%  							The figure is best viewed in color.
  							}
  \label{fig:scenarios}
	\vspace{-6mm}
\end{figure*}

\textVersion{}{We note that a possible implementation of \LB
may perform \emph{shortcutting}~\cite{GO07}
at the end of each iteration and use the cost of this shorter path instead of $c_{prev}$ (see Alg.~\ref{alg:MPLB} line~6).
Additionally, the techniques presented can be applied to a bidirectional version of FMT*~\cite{SSJP14}.
We did not present results incorporating these variants in order to focus on the fundamental characteristics that \LB introduces.}

\subsection{Fast convergence to high-quality solutions}

\begin{figure*}[t,b]
  \centering
  \subfloat
   [\sf Corridors]
   { 
   	\includegraphics[height = 3.1cm ]{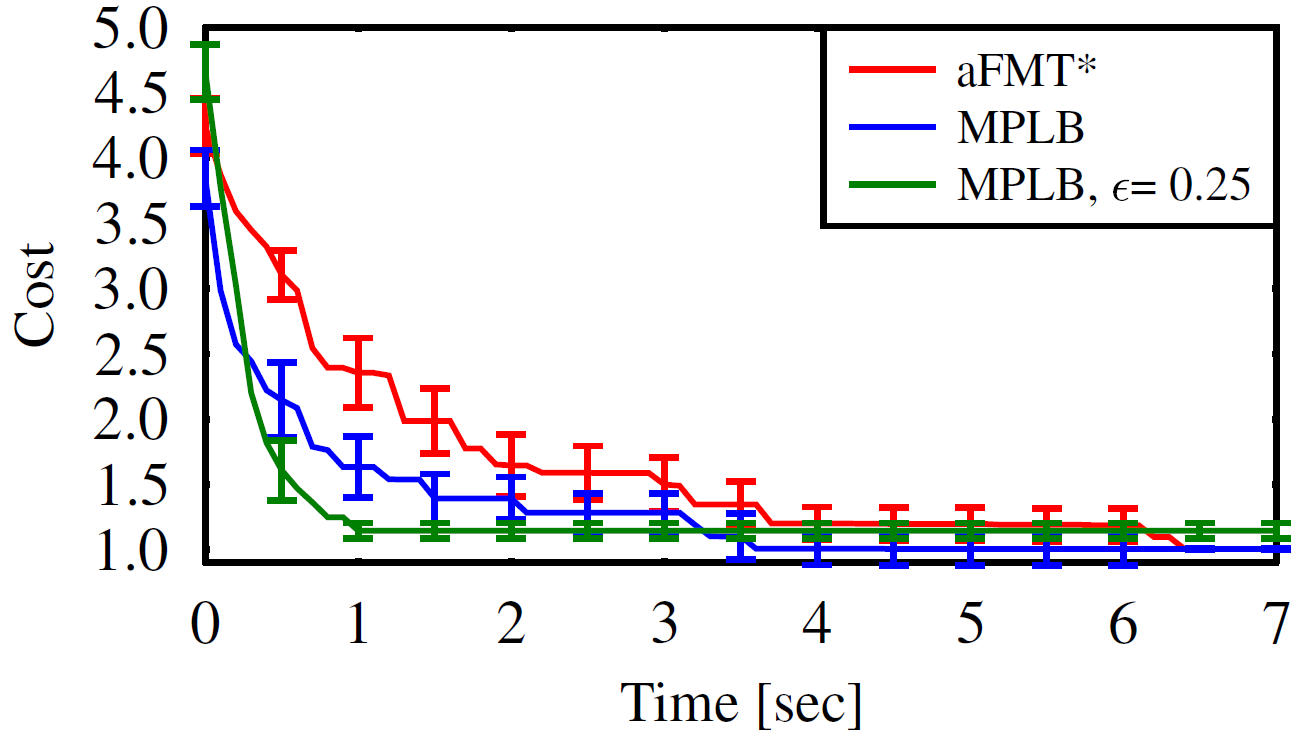}
   	\label{fig:corr_cost}
   }
  \subfloat
   [\sf Grids]
   { 
   	\includegraphics[height = 3.1cm]{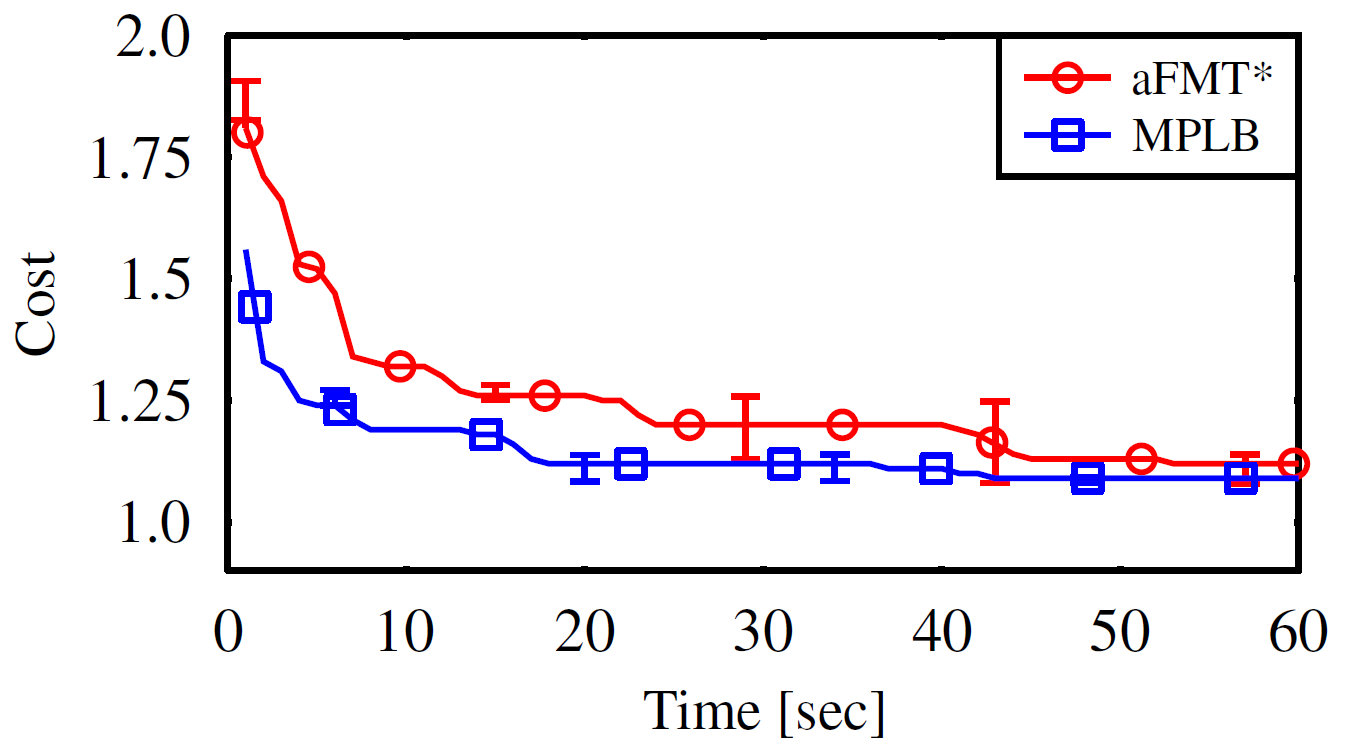}
   	\label{fig:grids_cost}
   }
   \subfloat
   [\sf Home]
   { 
   	\includegraphics[height = 3.1cm]{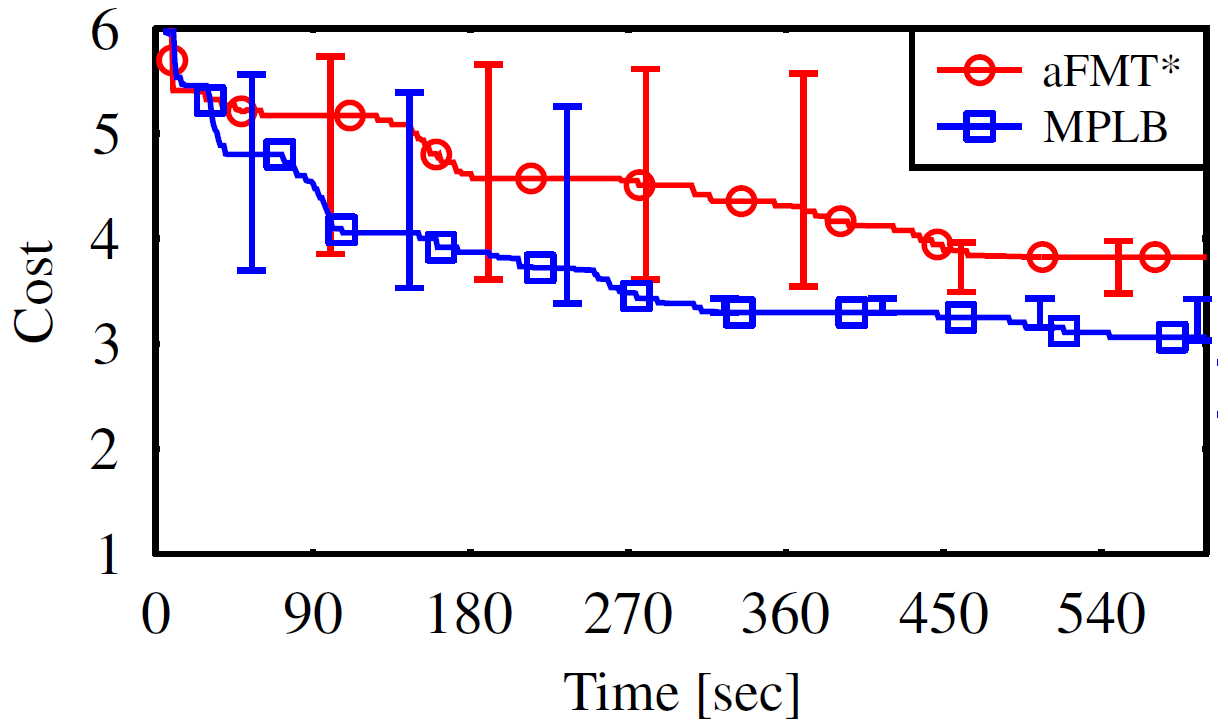}
   	\label{fig:cubicles_cost}
   }
  \caption{\sf 	\footnotesize
  							Average cost vs. time. 
  							Cost values are normalized such that a
  							cost of one represents the cost of an optimal path.
  							Low and high error bars denote the twentieth and eightieth 
  							percentile, respectively.}
  \label{fig:costs}
	\vspace{-6mm}
\end{figure*}

We start by comparing the cost of a solution obtained by aFMT* and \LB  as a function of time (Fig~\ref{fig:costs}). 
In all scenarios \LB typically finds a solution of given cost between two to three times faster than aFMT*.
In the Corridors scenario (Fig.~\ref{fig:corr}) the convergence rate can be sped up by using an approximation factor (see suggestion for future work in Section~\ref{sec:future}). 
%For the Grid and Cubicles scenarios (Figs.~\ref{fig:grids} and~\ref{fig:cubicles}) no benefit is gained when using an approximation factor as the first solution obtained by \LB is close to the optimal solution.
Interestingly, as we will show, the speed-up achieved by \LB is done while spending a smaller proportion of the time on LP compared to aFMT*.

\subsection{Nearest Neighbors and Local Planning calls}
\begin{figure}[t,b]
  \centering
  \subfloat
%   [\sf Cubicles]
   {
   	\includegraphics[height =3 cm]{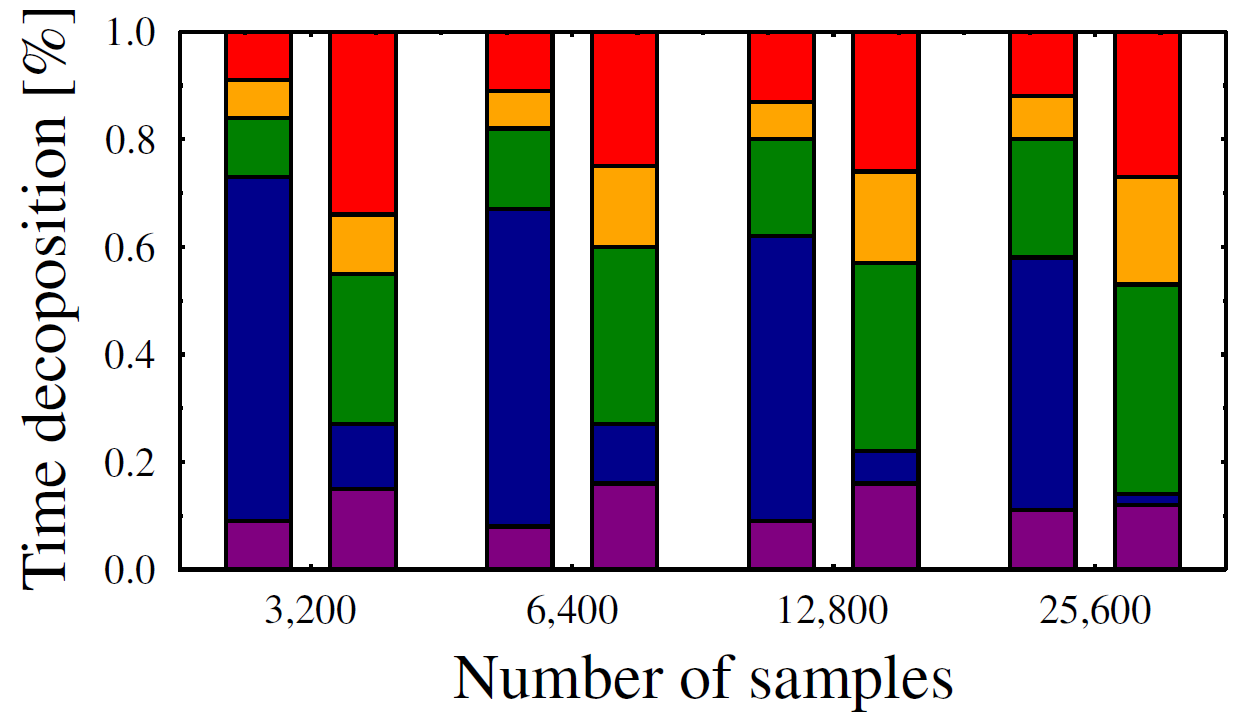}
   	\label{fig:grid_profiler}
   }
     \subfloat
%   [\sf Legend]
   {
   	\includegraphics[height =3 cm]{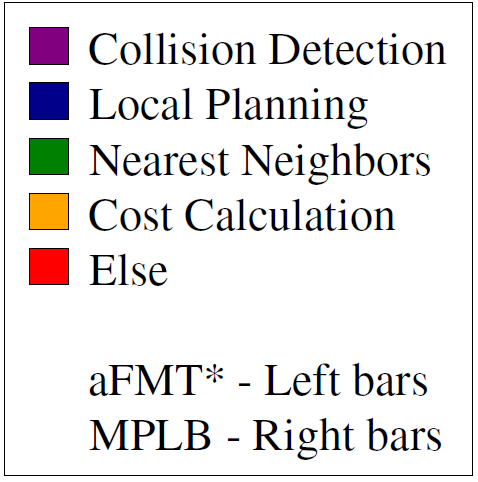}
   	\label{fig:legend}
   }
  \caption{\sf 	\footnotesize
  							Percentage of time spent for each of the main components in 
  							each iteration for both algorithms for the Grids Scenario.
  							Each iteration is represented by the number of samples used.
%  							The bars on the left (right) of each iteration represent the 
  							The left (right) bars of each iteration represent the 
  							result of aFMT* (\LB, respectively).
  							Note that the time of each iteration for each algorithm 
  							is different.}
  \label{fig:profiler}
	\vspace{-1mm}
\end{figure}

We profiled aFMT* and \LB and collected the total time spent on CD for point sampling, LP for edges, NN calls and cost computations.
Results for the Grids scenario are presented in Fig.~\ref{fig:profiler} 
(similar behavior was observed for the other scenarios as well).
Clearly,  CD computation time (due to sampling, not LP) is negligible for both algorithms and  cost calculation plays a larger (but still small) role for \LB.
CD calls due to LP calls are the main bottleneck for aFMT* 
(starting at around 65\% and gradually decreasing to 45\%). 
For \LB they start as a main time consumer but as samples are added their percentage of the overall iteration time becomes quite small 
(around 2\% for the last iteration).
NN calls play an almost complementary role to the LP and for the last iteration take 40\% of the total running time for the \LB algorithm 
while taking less than 20\% for aFMT*.

\begin{wraptable}{r}{4.9cm}
\vspace{-3mm}
\begin{tabular}{c|c|c}
$n$ & the ratio & the ratio \\
    &
$\frac{\#_{\texttt{NN}}(\text{\LB})}
  		{\#_{\texttt{NN}}(\text{aFMT}^*)}$ & 
$\frac{\#_{\texttt{LP}}(\text{\LB})}
 		  {\#_{\texttt{LP}}(\text{aFMT}^*)}$ \\
\hline
1.6K   &0.71 	& 0.38\\  
3.2K   &0.53 	& 0.31\\  
6.4K   &0.68  & 0.33\\  
12.8K  &0.68 	& 0.19\\  
25.6K  &0.69 	& 0.20\\  
51.2K  &0.99	& 0.05\\  
\end{tabular}
\vspace{-5mm}
\end{wraptable} 
The table to the right reports on the ratio of  NN and LP calls performed by \LB and aFMT* for the Grids scenario. 
The number of NN calls performed by \LB is lower than those performed by aFMT*.
As expected, \LB performs significantly less LP calls than aFMT*.

%\begin{wraptable}{r}{4.9cm}
%\vspace{-3mm}
%\begin{tabular}{c|c|c}
%$n$ & the ratio & the ratio \\
%    &
%$\frac{\#_{\texttt{NN}}(\text{\LB})}
%  		{\#_{\texttt{NN}}(\text{aFMT}^*)}$ & 
%$\frac{\#_{\texttt{LP}}(\text{\LB})}
% 		  {\#_{\texttt{LP}}(\text{aFMT}^*)}$ \\
%\hline
%1.6K   &1.06 & 0.67\\  
%3.2K   &1.22 & 0.62\\  
%6.4K   &1.55  & 0.56\\  
%12.8K  &0.99 & 0.37\\  
%25.6K  &0.99 & 0.49\\  
%51.2K  &0.99		& 0.01\\  
%\end{tabular}
%\vspace{-5mm}
%\end{wraptable} 
%The table to the right reports on the ratio of  NN and LP calls performed by \LB and aFMT* for the Cubicles scenario. 
%The number of NN calls performed by \LB is higher than aFMT* at times but the excess is almost negligible.
%As expected, \LB performs significantly less LP calls than aFMT*.

%% file: tex/Discussion.tex
\section{Conclusion and outlook}
\label{sec:future}
In this work we show that by using effective lower bounds and with  no compromise on the cost of paths produced by the algorithm, the weight of CD (via LP calls) may become almost negligible with respect to NN calls.
This follows the ideas presented by Bialkowski et al.~\cite{BKOF12} but uses different, more general, methods.
Looking into NN computation, one can notice that AO algorithms such as sPRM*~\cite{KF11}, FMT* and \LB rely on a specific type of NN computation: 
given a set $P$ of $n$ points, 
either compute 
for each point all its $k$ nearest neighbors, 
or all neighbors within distance $r$ from the point.
In both cases, $P$ is known in advance 
%(i.e. there are no ``general'' queries to this NN problem)
and
$k$ (or $r$) are parameters that do not change throughout the algorithm or throughout a single iteration of the algorithm.

\textVersion
{This calls for using application-specific NN algorithms and not general purpose ones.
For example, 
the recent work on randomly shifted grids by Aiger et al.~\cite{AKS13} may be used.
Indeed, we show that using this data structure allows to significantly speed up motion-planning algorithms~\cite{KSH14}.
}
{This calls for using application-specific NN algorithms and not necessarily  general purpose ones (such as $kd$-trees~\cite{FBF77}).
Such algorithms exist, for example, in high dimensions locally sensitive hashing (LSH)~\cite{IM98} may be used.
For smaller dimensions, the work by Aiger et al.~\cite{AKS13} may be a good, practical choice.}

\textVersion
{A different possibility to enhance \LB is to relax AO to ANO:
Asymptotically-optimal motion-planning algorithms such as MPLB often, from a certain stage of their execution,
invest huge computational resources at only slightly improving the cost of the current best existing solution. 
Similar to the approach presented by the authors in a previous work~\cite{SH13} one can construct a variant such that given an \emph{approximation factor}~$\varepsilon$, the cost of the solution obtained  is within a factor of 
$1 + \varepsilon$ from the solution that MPLB would obtain for the same set of samples. 
We expand on this idea in the extended version of our paper~\cite{SH14}.
Preliminary results, presented in Fig.~\ref{fig:corr_cost} show the potential benefit of this approach.
}
{Directions for further research include continuing our theoretical comparative analysis with aFMT* by focusing on the number of NN calls performed by both algorithms.}

%% file: tex/thanks.tex
\section{Acknowledgements}
We wish to thank Marco Pavone and his co-workers for their advice and support regarding the FMT* algorithm.
\textVersion{}{Additionally we wish to thank Kiril Solovey for fruitful discussions and insightful comments regarding this work.}

%% file: MPLB.bbl
\begin{thebibliography}{10}
\providecommand{\url}[1]{#1}
\csname url@samestyle\endcsname
\providecommand{\newblock}{\relax}
\providecommand{\bibinfo}[2]{#2}
\providecommand{\BIBentrySTDinterwordspacing}{\spaceskip=0pt\relax}
\providecommand{\BIBentryALTinterwordstretchfactor}{4}
\providecommand{\BIBentryALTinterwordspacing}{\spaceskip=\fontdimen2\font plus
\BIBentryALTinterwordstretchfactor\fontdimen3\font minus
  \fontdimen4\font\relax}
\providecommand{\BIBforeignlanguage}[2]{{%
\expandafter\ifx\csname l@#1\endcsname\relax
\typeout{** WARNING: IEEEtran.bst: No hyphenation pattern has been}%
\typeout{** loaded for the language `#1'. Using the pattern for}%
\typeout{** the default language instead.}%
\else
\language=\csname l@#1\endcsname
\fi
#2}}
\providecommand{\BIBdecl}{\relax}
\BIBdecl

\bibitem{JP13}
L.~Janson and M.~Pavone, ``Fast marching trees: a fast marching sampling-based
  method for optimal motion planning in many dimensions,'' \emph{CoRR}, vol.
  abs/1306.3532, 2013.

\bibitem{CBHKKLT05}
H.~Choset, K.~M. Lynch, S.~Hutchinson, G.~Kantor, W.~Burgard, L.~E. Kavraki,
  and S.~Thrun, \emph{Principles of Robot Motion: Theory, Algorithms, and
  Implementation}.\hskip 1em plus 0.5em minus 0.4em\relax MIT Press, June 2005.

\bibitem{BKOF12}
J.~Bialkowski, S.~Karaman, M.~Otte, and E.~Frazzoli, ``Efficient collision
  checking in sampling-based motion planning,'' in \emph{WAFR}, 2012, pp.
  365--380.

\bibitem{KF11}
S.~Karaman and E.~Frazzoli, ``Sampling-based algorithms for optimal motion
  planning,'' \emph{I. J. Robotic Res.}, vol.~30, no.~7, pp. 846--894, 2011.

\bibitem{KSLO96}
L.~E. Kavraki, P.~Svestka, J.-C. Latombe, and M.~H. Overmars, ``Probabilistic
  roadmaps for path planning in high dimensional configuration spaces,''
  \emph{IEEE Trans. Robot.}, vol.~12, no.~4, pp. 566--580, 1996.

\bibitem{KL00}
J.~J. Kuffner and S.~M. LaValle, ``{RRT}-{C}onnect: An efficient approach to
  single-query path planning,'' in \emph{ICRA}, 2000, pp. 995--1001.

\bibitem{AS11}
B.~Akgun and M.~Stilman, ``Sampling heuristics for optimal motion planning in
  high dimensions,'' in \emph{IROS}, 2011, pp. 2640--2645.

\bibitem{GSB14}
J.~D. Gammell, S.~S. Srinivasa, and T.~D. Barfoot, ``Informed {RRT}*: Optimal
  incremental path planning focused through an admissible ellipsoidal
  heuristic,'' in \emph{IROS}, 2014, to appear.

\bibitem{DB14}
A.~Dobson and K.~E. Bekris, ``Sparse roadmap spanners for asymptotically
  near-optimal motion planning,'' \emph{I. J. Robotic Res.}, vol.~33, no.~1,
  pp. 18--47, 2014.

\bibitem{LLB13}
Z.~Littlefield, Y.~Li, and K.~E. Bekris, ``Efficient sampling-based motion
  planning with asymptotic near-optimality guarantees for systems with
  dynamics,'' in \emph{IROS}, 2013, pp. 1779--1785.

\bibitem{SH13}
O.~Salzman and D.~Halperin, ``Asymptotically near-optimal {RRT} for fast,
  high-quality, motion planning,'' in \emph{ICRA}, 2014, pp. 4680--4685.

\bibitem{AT13}
O.~Arslan and P.~Tsiotras, ``Use of relaxation methods in sampling-based
  algorithms for optimal motion planning,'' in \emph{ICRA}, 2013, pp.
  2421--2428.

\bibitem{K04}
J.~J. Kuffner, ``Effective sampling and distance metrics for 3d rigid body path
  planning,'' in \emph{ICRA}, 2004, pp. 3993--3998.

\bibitem{JP12}
L.~Jaillet and J.~M. Porta, ``Asymptotically-optimal path planning on
  manifolds,'' in \emph{RSS}, 2012.

\bibitem{WBC13}
W.~Wang, D.~Balkcom, and A.~Chakrabarti, ``A fast streaming spanner algorithm
  for incrementally constructing sparse roadmaps,'' \emph{IROS}, pp.
  1257--1263, 2013.

\bibitem{SH14}
O.~Salzman and D.~Halperin, ``Asymptotically-optimal motion planning using
  lower bounds on cost,'' \emph{CoRR}, vol. abs/1403.7714, 2014.

\bibitem{P84}
J.~Pearl, \emph{Heuristics: Intelligent Search Strategies for Computer Problem
  Solving}.\hskip 1em plus 0.5em minus 0.4em\relax Addison-Wesley, 1984.

\bibitem{SMK12}
I.~A. {\c{S}}ucan, M.~Moll, and L.~E. Kavraki, ``The {O}pen {M}otion {P}lanning
  {L}ibrary,'' \emph{{IEEE} Robot. Automat. Mag.}, vol.~19, no.~4, pp. 72--82,
  2012.

\bibitem{AKS13}
D.~Aiger, H.~Kaplan, and M.~Sharir, ``Reporting neighbors in high-dimensional
  euclidean space,'' \emph{{SIAM} J. Comput.}, vol.~43, no.~4, pp. 1363--1395,
  2014.

\bibitem{KSH14}
M.~Kleinbort, O.~Salzman, and D.~Halperin, ``Efficient high-quality motion
  planning by fast all-pairs $r$-nearest-neighbors,'' \emph{CoRR}, vol.
  abs/1409.8112, 2014.

\end{thebibliography}
